\documentclass{article}

\newcommand{\ourtechnique}{disentangled influence audits}
\newcommand{\wrapmodel}{disentangled model}

\usepackage{arxiv}

\usepackage[utf8]{inputenc} % allow utf-8 input
\usepackage[T1]{fontenc}    % use 8-bit T1 fonts
\usepackage{amsmath,amsthm,amscd}

\newtheorem{prop}{Proposition}
\usepackage{hyperref}       % hyperlinks
\usepackage{url}            % simple URL typesetting
\usepackage{booktabs}       % professional-quality tables
\usepackage{amsfonts}       % blackboard math symbols
\usepackage{nicefrac}       % compact symbols for 1/2, etc.
\usepackage{microtype}      % microtypography
\usepackage{xcolor}
\usepackage{graphicx}
\usepackage{wrapfig}
\usepackage{multicol}
\usepackage{clrscode3e}
\usepackage[final,inline,nomargin,index]{fixme}
\fxsetup{theme=color,mode=multiuser,inlineface=\itshape,envface=\itshape}
\FXRegisterAuthor{sv}{asv}{\colorbox{gray!10!white}{\color{black}Suresh}}
\FXRegisterAuthor{sf}{asf}{\colorbox{blue!10!white}{\color{black}Sorelle}}
\FXRegisterAuthor{cs}{acs}{\colorbox{red!10!white}{\color{black}Carlos}}
\FXRegisterAuthor{cm}{acm}{\colorbox{orange!10!white}{\color{black}Charlie}}
\FXRegisterAuthor{rp}{arp}{\colorbox{orange!10!white}{\color{black}Richard}}
\fxusetargetlayout{color}

\newcommand{\reals}{\mathbb{R}}

\newcommand{\shap}{\texttt{shap}}
\newcommand{\bba}{\texttt{BBA}}
\newcommand{\mse}{\operatorname{MSE}}

\makeatletter
\def\blfootnote{\xdef\@thefnmark{}\@footnotetext}
\makeatother

\title{Disentangling Influence: Using disentangled representations to audit model predictions}

\author{%
   Charles T. Marx \\
   Haverford College \\
   \texttt{cmarx@haverford.edu} \\
   \And
   Richard Lanas Phillips \\
   Cornell University \\
   \texttt{rlp246@cornell.edu} \\
   \And
   Sorelle A. Friedler \\
   Haverford College \\
   \texttt{sorelle@cs.haverford.edu} \\
   \AND
   Carlos Scheidegger \\
   University of Arizona \\
   \texttt{cscheid@cs.arizona.edu} \\
    \And
   Suresh Venkatasubramanian \\
   University of Utah \\
   \texttt{suresh@cs.utah.edu}
}
\begin{document}
\blfootnote{Partially supported by the NSF under grants IIS-1513651, IIS-1633724, IIS-1633387 and DMR-1709351, the DARPA SD2 program, and the Arnold and Mabel Beckman Foundation.  The Titan Xp GPU used for this research was donated by the NVIDIA Corporation.  Code can be found at \url{https://github.com/charliemarx/disentangling-influence}}

\maketitle

\begin{abstract}
  Motivated by the need to audit complex and black box models, there has been
  extensive research on quantifying how data features influence model
  predictions. Feature influence can be direct (a direct
  influence on model outcomes) and indirect (model outcomes are influenced via
  proxy features). Feature influence can also be expressed in aggregate over the
  training or test data or locally with respect to a single point. Current
  research has typically focused on one of each of these dimensions. 

  In this paper, we develop \emph{\ourtechnique}, a procedure
  to audit the indirect influence of features. 
  Specifically, we show that disentangled representations provide a mechanism
  to identify proxy features in the dataset, while allowing an explicit
  computation of feature influence on either individual outcomes or
  aggregate-level outcomes. 
  We show through both theory and experiments that \ourtechnique\ can both
  detect proxy features and show, for each individual or in aggregate, which of these proxy
  features affects the classifier being audited the most. In this respect, our
  method is more powerful than existing methods for ascertaining feature
  influence. 
\end{abstract}

\section{Introduction}

As machine learning models have become increasingly complex, there has been a
growing subfield of work on interpreting and explaining the predictions of these
models \cite{molnar2018interpretable, guidotti2018survey}.  In order to assess
the importance of particular features to aggregated model predictions or
outcomes for an individual instance, a variety of direct and indirect
\emph{feature influence} techniques have been developed.  While direct feature
influence \cite{datta2016, Henelius2014BlackBox, lundberg2017unified, lime}
focuses on determining the importance of features used directly by the model to
determine an outcome, \emph{indirect feature influence} techniques
\cite{adler2018auditing} report that a feature is important if that feature
\emph{or a proxy} had an influence on the model outcomes. 

Feature influence methods can focus on the influence of a feature taken over all
instances in the training or test set \cite{datta2016, adler2018auditing}, or on
the local feature influence on a single \emph{individual} item of the training
or test set \cite{lime, lundberg2017unified} (both of which are different than
the influence of a specific training instance on a model's parameters
\cite{koh2017understanding}).  Both the global perspective given by considering
the impact of a feature on all training and/or test instances as well as the
local, individual perspective can be useful when auditing a model's predictions.
Consider, for example, the question of fairness in an automated hiring decision:
determining the indirect influence of gender on all test outcomes could help us
understand whether the system had disparate impacts overall, while an
individual-level feature audit could help determine if a specific person's
decisions were due in part to their gender.\footnote{While unrelated to feature
  influence, the idea of \emph{recourse} \cite{ustun2019actionable} also emphasizes the
  importance of individual-level explanations of an outcome or how to change it.}

\paragraph{Our Work.}

In this paper we present a general technique to perform both global and
individual-level indirect influence audits. Our technique is \emph{modular} --
it solves the indirect influence problem by reduction to a \emph{direct
  influence} problem, allowing us to benefit from existing techniques. 

Our key insight is that \emph{disentangled representations} can be used
to do indirect influence computation. The idea of a disentangled representation
is to learn independent factors of variation that reflect the natural symmetries
of a data set. This approach has been very successful in generating
representations in deep learning that can be manipulated while creating
realistic inputs \cite{AlemiFD016,bengio2013representation, esmaeili19a,  KumarVID, tschannen2018recent}. Related methods use competitive learning to ensure a representation is free of protected information while preserving other information \cite{storkey2016,laftr}. 

In our context, the idea is to \emph{disentangle} the
influence of the feature whose (indirect) influence we want to compute. By doing
this, we obtain a representation in which we can manipulate the feature directly
to estimate its influence.
Our approach has a number of advantages. We can connect indirect influence in the native representation to
  direct influence in the disentangled representation. Our method creates a
  \emph{\wrapmodel}: a wrapper to the original model with the disentangled
  features as inputs. This implies that it works for (almost) any model for
  which direct influence methods work, and also allows us to use any future developed direct influence model. 

Specifically, our \ourtechnique\ approach provides the following contributions:

\begin{enumerate}
\item Theoretical and experimental justification that the \wrapmodel\ and associated \ourtechnique\ we create provides an accurate indirect influence audit of complex, and potentially black box, models.

\item Quality measures, based on the error of the disentanglement and the error of the reconstruction of the original input, that can be associated with the audit results.

\item An indirect influence method that can work in association with both global and individual-level feature influence mechanisms.  Our \ourtechnique\ can additionally audit continuous features and image data; types of audits that were not possible with previous indirect audit methods (without additional preprocessing).

\end{enumerate}

\section{Our Methodology}
\label{sec:theory}

\subsection{Theoretical background}
\label{sec:formalisms}

Let $P$ and $X$ denote sets of attributes with associated domains
$\mathcal{P}$ and $\mathcal{X}$. $P$ represents \emph{features of
  interest}: these could be protected attributes of the data or any other
features whose influence we wish to determine. For convenience
we will assume that $\mathcal{P}$ consists of the values taken by a single
feature -- our exposition and techniques work more generally. $X$ represents other attributes of the data that may or
may not be influenced by features in $P$. An \emph{instance} is thus a point
$(p,x) \in \mathcal{P} \times \mathcal{X}$. Let $Y$ denote the space of
\emph{labels} for a learning task ($Y = \{+1,-1\}$ for binary
classification or $\reals$ for regression).

\paragraph{Disentangled Representation.}

Our goal is to find an alternate representation of an instance
$(p,x)$. Specifically, we would like to construct $x' \in \mathcal{X'}$ that
represents all factors of variation that are \emph{independent} of $P$, as well
as a mapping $f$ such that $f(p,x) = (p, x')$. We will refer to the associated
new domain as $\mathcal{D}' = \mathcal{P} \times \mathcal{X}'$. We can formalize this using the
framework of \cite{higgins2018towards}. For any $(p,x)$, we can define a group action implicitly in terms of its orbits: specifically, we define an equivalence
relation $(p,x_1) \equiv (p', x_2)$ if in the underlying data, changing $p$ to $p'$
would change $x_1$ to $x_2$. Note that this is an orbit with respect to the
permutation group $S_m$ on $\mathcal{P}$ (where $m = |\mathcal{P}|$ is the size
of the domain $\mathcal{P}$). Our goal is to find an equivariant function $f$
and an associated group action that yields the desired disentangled
representation. 

We can define a group action $\circ : S_m \times \mathcal{D}' \rightarrow \mathcal{D}'$ on the \emph{disentangled} representation $(p,x')$ as the
mapping
$ \pi \circ (p, x') = (\pi(p), x') $.
Then, given $f$ such that $f(x) = x'$, it is
equivariant $(\pi \circ (p, f(x)) = f(\pi \circ(p, x))$ and the representation $(p, x')$ satisfies the property of being
disentangled. Formally, the group action is the product of
$S_m$ and the identity mapping on $x'$, but for clarity we omit this detail. 

\[\begin{CD} 
(p,x) @>\pi >> (\pi(p),x) \\ 
@V f VV @V f VV \\
(p,x') @>\pi>> (\pi(p), x')
\end{CD}\]

\paragraph{Direct and indirect influence}
Given a model $M : \mathcal{D} \to Y$, a \emph{direct influence} measure
quantifies the degree to which any particular feature influences the outcome of
$M$ on a specific input.  In this paper, we use the SHAP values proposed by
\cite{lundberg2017unified} that are inspired by the Shapley values in game
theory. For a model $M$ and input $x$, the influence of feature $i$ is defined
as \cite[Eq. 8]{lundberg2017unified}
$ \phi_{i}(M, x)=\sum_{z \subseteq x}
  \frac{\left|z\right| !\left(n-\left|z\right|-1\right) !}{n
    !}\left[M_{x}\left(z\right)-M_{x}\left(z \setminus
      i\right)\right] $
where $\|z\|$ denotes the number of nonzero entries in $z$, $z \subseteq x$ is a
vector whose nonzero entries are a subset of the nonzero entries in $x$, $z
\setminus i$ denotes $z$ with the $i^{th}$ entry set to zero, and $n$ is the
number of features. Finally, $M_x(z) = E[M(z)|z_S]$, the conditional expected
value of the model subject to fixing all the nonzero entries of $z$ ($S$ is the
set of nonzero entries in $z$).

\emph{Indirect influence} attempts to capture how a feature might influence the
outcome of a model even if it is not explicitly represented in the data, i.e its
influence is via \emph{proxy} features. The above direct influence measure
cannot capture these effects because the information encoded in a feature $i$ might
be retained in other features even if $i$ is removed. 
We say that the \emph{indirect influence} of feature $i$ on the outcome of
model $M$ on input $x$ is the direct influence of some \emph{proxy} for $i$,
where a proxy for $i$ consists of a set of features $S$ and a
function $g$ that \emph{predicts} $i$: i.e such that $g(x_S) \simeq
x_{(i)}$. Note that this generalizes in particular the notion of indirect
influence defined by \cite{adler2018auditing}: in their work, indirect
influence is defined via an explicit attempt to first remove any possible proxy
for $i$ and then evaluate the direct influence of $i$. Further, note that if
there are no features that can predict $x_i$, then the indirect and direct
influence of $i$ are the same (because the only proxy for $i$ is itself). 

\paragraph{Disentangled influence}
The key insight in our work is that disentangled representations can be used to
compute indirect influence. Assume that we have an initial representation of a
feature vector as $(p,x)$ and train a model $M$ on labeled pairs $((p,x),
y)$. Our goal is to determine the indirect influence of $p$ on the model
outcome. Suppose we construct a disentangled representation $(p,x')$ as defined
above, with the associated encoding function $f(x) = x'$ and decoding function
$f^{-1}$.

\begin{prop}
  The indirect influence of $p$ on the outcome of $M$ on input $x$ equals
  $\phi_p(M', x')$, where $M' = f^{-1} \circ M$. 
\end{prop}
\begin{proof}
  By the properties of the disentangled representation, there is no proxy for
  $p$ in the components of $x'$: if there were, then it would not be true that
  the $f$ was equivariant (because we could not factor the action on $p$
  separately from the identity mapping on $x'$).

  Thus, if we wished to compute the indirect influence of $p$ on model $M$ with
  outcome $x$, it is sufficient to compute the direct influence of $p$ on the
  model that first converts from the disentangled representation back to the
  original representation and then applies $M$. 
\end{proof}

\paragraph{Dealing with errors.}

The above proposition holds if we are able to obtain a perfectly disentangled and
invertible representation. In practice, this might not be the case and the
resulting representation might introduce errors. In particular, assume that our
decoder function is some $g \not= f^{-1}$. While we do not provide an  explicit formula for the
dependence of the influence function parameters, we  note that it is a linear
function of the predictions, and so we can begin to understand the errors in the
influence estimates by looking at the behavior of the predictor 
with respect to $p$. 

Model output can be written as $\hat{y} = (M \circ g)(p, x')$. Recalling that
$g(p, x') = (p, \hat{x})$, the partial derivative of $\hat{y}$ with
respect to $p$ can be written as
$  \frac{\partial \widehat y}{\partial p} %&
 =\frac{\partial (M \circ g)}{\partial
    p} %\\
  %&
   =\frac{\partial M}{\partial \hat{x}} \frac{\partial \hat{x}}{\partial p} +
    \frac{\partial M}{\partial p} %\\
  %&
  = \frac{\partial M}{\partial \hat{x}} \frac{\partial \hat{x}}{\partial x'}
    \frac{\partial x'}{\partial p} +
    \frac{\partial M}{\partial p} 
    $.
Consider the term $\frac{\partial x'}{\partial p}$. If the disentangled
representation is perfect, then this term is zero (because $x'$ is unaffected by
$p$), and therefore we get  $\frac{\partial \widehat y}{\partial p} =
\frac{\partial M}{\partial p}$ which is as we would expect. If the
reconstruction is perfect (but not necessarily the disentangling), then the term
$\frac{\partial \hat{x}}{\partial x'}$ is $1$. What remains is the partial
derivative of $M$ with respect to the latent encoding $(x', p)$.

\subsection{Implementation}
\label{sec:methodology}

Our overall process requires two separate pieces: 1) a method to create
disentangled representations, and 2) a method to audit direct features. 
In most experiments in this paper, we use adversarial autoencoders \cite{makhzani2015adversarial} to generate
disentangled representations, and Shapley values from the \shap\
technique for auditing direct features \cite{lundberg2017unified} (as described
above in Section~\ref{sec:formalisms}). 

\paragraph*{Disentangled representations via adversarial autoencoders}
We create disentangled representations by training three separate neural networks, which we denote $f$, $g$, and $h$ (see Figure~\ref{fig:system}).
Networks $f$ and $g$ are %typical 
autoencoders: the image of $f$ has lower dimensionality than the domain of $f$, and the training process seeks for $g \circ f$ to be an approximate identity, through gradient descent on the reconstruction error $||(g \circ f)(x) - x||$.
Unlike regular autoencoders, $g$ is also given direct access to the protected attribute.
\emph{Adversarial} autoencoders \cite{makhzani2015adversarial}, in addition, use an ancillary network $h$ that attempts to recover the protected attribute from the image of $f$, \emph{without access to $p$ itself}.
(Note the slight abuse of notation here: $h$ is assumed not to have access to $p$, while $g$ does have access to it.)
During the training of $f$ and $g$, we seek to reduce $||(g \circ f)(x) - x||$, but also to \emph{increase the error} of the discriminator $h \circ f$.
\begin{wrapfigure}{r}{0.5\textwidth}
\begin{center}
\includegraphics[width=2.5in]{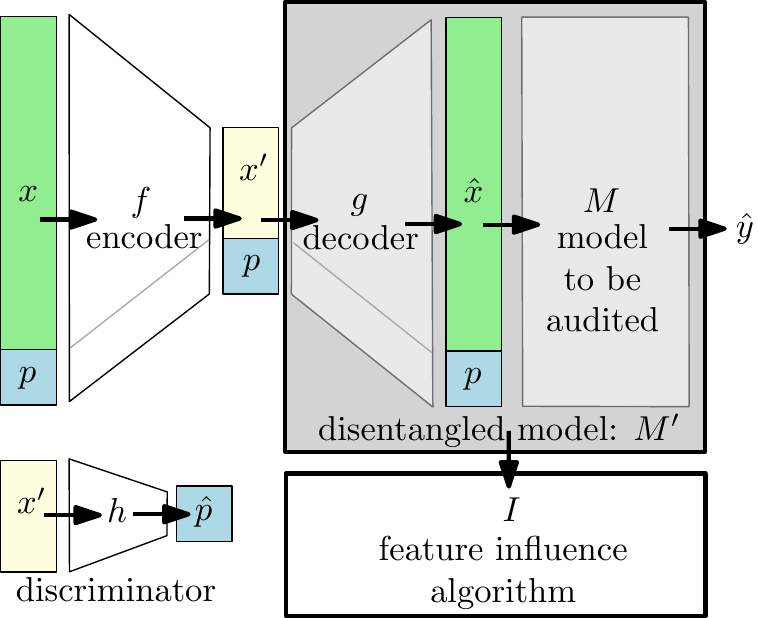}
\caption{System diagram when auditing the indirect influence of feature $p$ on the outcomes of model $g$ for instance $x$ using direct influence algorithm $I$.}
\vspace{-.2in}
\label{fig:system}
\end{center}
\end{wrapfigure}
The optimization process of $h$ tries to recover the protected attribute from the code generated by $f$.
($h$ and $f$ are the \emph{adversaries}.)
When the process converges to an equilibrium, the code generated by $f$ will contain no information about $p$ that is useful to $h$, but $g \circ f$ still reconstructs the original data correctly: $f$ disentangles $p$ from the other features.

The loss functions used to codify this process are
$\mathcal L_{\texttt {Enc}} = \mse(x, \hat{x}) - \beta \mse(p, \hat{p})$, ~
$\mathcal L_{\texttt {Dec}} = \mse(x, \hat{x})$,~ and
$\mathcal L_{\texttt {Disc}} = \mse(p, \hat{p})$,
where $\mse$ is the mean squared error and $\beta$ is a hyperparameter determining the importance of disentanglement relative to reconstruction. When $p$ is a binary feature, $\mathcal L_{\texttt {Enc}}$ and $\mathcal L_{\texttt {Disc}}$ are adjusted to use binary cross entropy loss between $p$ and $\hat{p}$.

\paragraph*{Disentangled feature audits} Concretely, our method works as follows, where the variable names match the diagram in Figure~\ref{fig:system}:

\begin{codebox}
  \Procname{ $\proc{Disentangled-Influence-Audit}(X, M)$}
  \li \For $p$ \kw{in} \proc{Features}(X)
  \li \Do
  $(\id{f}, \id{g}, \id{h}) \gets \proc{Disentangled-Representation}(X, p)$ \Comment ($\id{h}$ is not used)
  \li $M' \gets g \circ M$
  \li $X' \gets \{ f(x)\ \For\ x\ \kw{in}\ X \} $
  \li $\proc{Shap}_p \gets \proc{Direct-Influence}(X', p, M')$
\End

\li \Return $\{ \proc{Shap}_p\ \For\ p\ \kw{in}\ \proc{Features}(X) \}$
\end{codebox}

We note here one important difference in the interpretation of disentangled influence values when contrasted with regular Shapley values.
Because the influence of each feature is determined on a \emph{different} disentangled model, the scores we get are not directly interpretable as a partition of the model's prediction.
For example, consider a dataset in which feature $p_1$ is responsible for 50\% of the direct influence, while feature $p_2$ is a perfect proxy for $p_1$, but shows 0\% influence under a direct audit. %This ``excess influence'' is slightly unintuitive, but
Relative judgments of feature importance remain sensible.

\section{Experiments}
\label{sec:experiments}

In this section, we'll assess the extent to which the \ourtechnique\ is able to identify sources of indirect influence to a model and quantify its error. All data and code for the described method and below experiments is available at \texttt{https://github.com/charliemarx/disentangling-influence}.

\subsection{Synthetic $x+y$ Regression Data}

In order to evaluate whether the indirect influence calculated by the \ourtechnique\ correctly captures all influence of individual-level features on an outcome, we will consider influence on a simple synthetic $x+y$ dataset.  It includes 5,000 instances two variables $x$ and $y$ drawn independently from a uniform distribution over $[0,1]$ that are added to determine the label $x+y$.  It also includes proxy variables $2x$, $x^2$, $2y$, and $y^2$.  A random noise variable $c$ is also included that is drawn independently of $x$ and $y$ uniformly from $[0,1]$.  The model we are auditing is a handcrafted model that contains no hidden layers and has fixed weights of 1 corresponding to $x$ and $y$ and weights of 0 for all other features (i.e., it directly computes $x+y$).  
We use \shap\ as the direct influence delegate method \cite{lundberg2017unified}.\footnote{This method is available via \texttt{pip install shap}.  See also: \url{https://github.com/slundberg/shap}}

In order to examine the impact of the quality of the disentangled representation on the results, we considered both a handcrafted disentangled representation and a learned one.  For the former, nine unique models were handcrafted to disentangle each of the nine features perfectly (see Appendix A for details).   
The learned disentangled representation is created according to the adversarial autoencoder methodology described in more detail in the previous section.

\begin{figure}[htb]
\begin{center}
\begin{tabular}{ccc}

& \textbf{Direct Influence~~~~} & \textbf{Indirect Influence~~~~~~~~~~}\\

\rotatebox{90}{~~~~~Handcrafted DR}
& \includegraphics[width=2in]{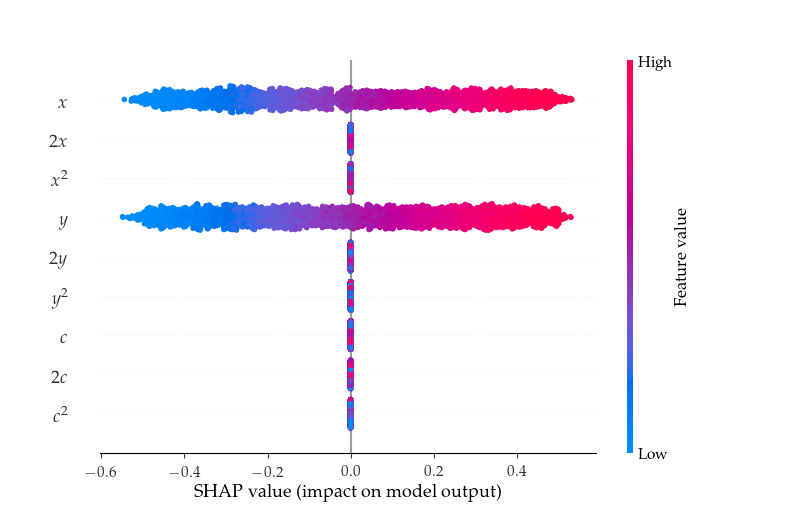}
&
\includegraphics[width=2in]{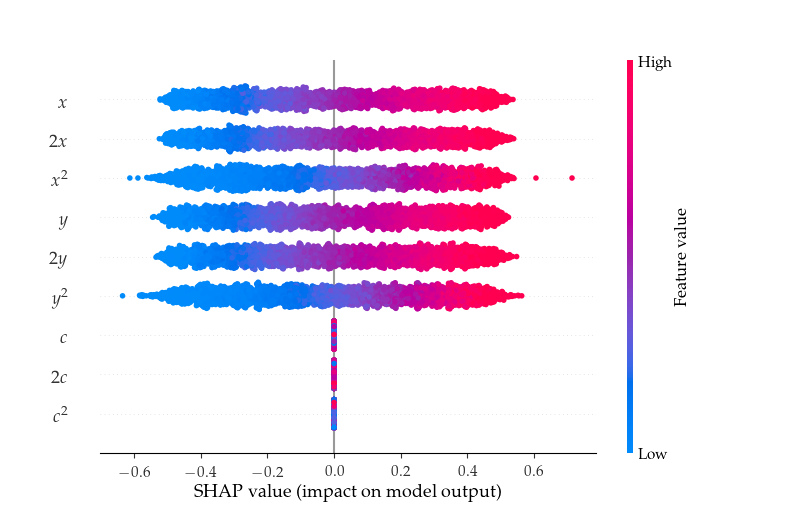}\\

\rotatebox{90}{~~~~~~~~~~Learned DR}
& \includegraphics[width=2in]{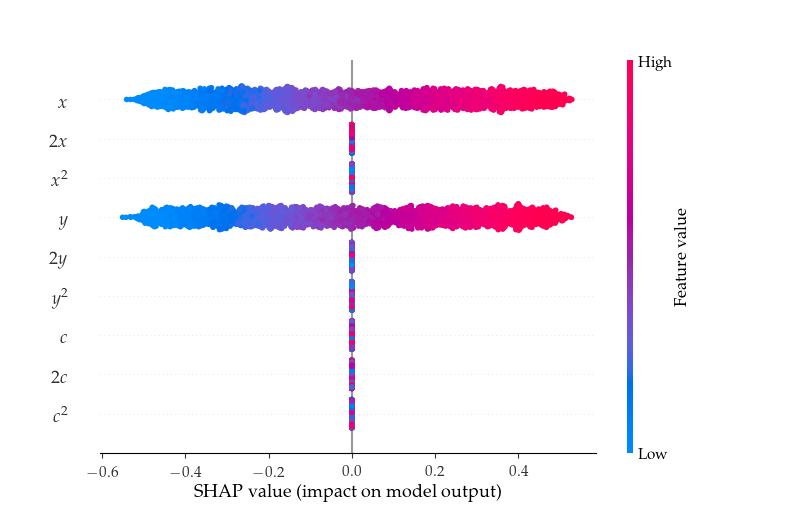}
&
\includegraphics[width=2in]{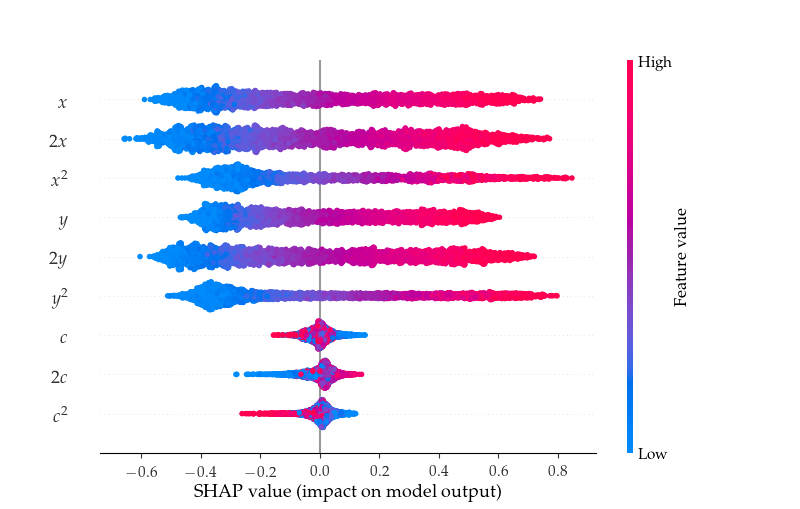}

\end{tabular}
\caption{Synthetic $x+y$ data direct \shap\ (left) and indirect (right) feature influences using a handcrafted (top row) or learned disentangled representation (bottom row).}
\label{fig:synth}
\end{center}
\end{figure}

The results for the handcrafted disentangled representation (top of Figure \ref{fig:synth}) are as expected: features $x$ and $y$ are the only ones with direct influence, all $x$ or $y$ based features have the same amount of indirect influence, while all features including $c$ have zero influence.  Using the learned disentangled representation introduces the potential for error: the resulting influences (bottom of Figure \ref{fig:synth}) show more variation between features, but the same general trends as in the handcrafted test case. 

Additionally, note that since \shap\ gives influence results per individual instance, we can also see that (for both models) instances with larger (or, respectively, smaller) $2x$ or $2y$ values give larger (respectively, smaller) results for the label $x + y$, i.e., have larger absolute influences on the outcomes.

\subsubsection{Error Analyses}

There are two main sources of error for \ourtechnique: error in the reconstruction of the original input $x$ and error in the disentanglement of $p$ from $x'$ such that the discriminator is able to accurately predict some $\hat{p}$ close to $p$.  We will measure the former error in two ways.  First, we will consider the \emph{reconstruction error}, which we define as $x - \hat{x}$.  Second, we consider the \emph{prediction error}, which is $g(x) - g(\hat{x})$ - a measure of the impact of the reconstruction error on the model to be audited.  Reconstruction and prediction errors close to 0 indicate that the \wrapmodel\ $M'$ is similar to the model $M$ being audited.  We measure the latter form of error, the \emph{disentanglement error}, as $\frac{1}{n} \sum_{i=1}^n (p - \hat{p})^2 / var(p)$ where $var(p)$ is the variance of $p$.  A disentanglement error of below $1$ indicates that information about that feature may have been revealed, i.e., that there may be indirect influence that is not accounted for in the resulting influence score.
In addition to the usefulness of these error measures during training time, they also provide information that helps us to assess the quality of the indirect influence audit, including at the level of the error for an individual instance.

\begin{figure}[htbp]
\begin{center}
\begin{tabular}{ccc}
\includegraphics[width=1.7in]{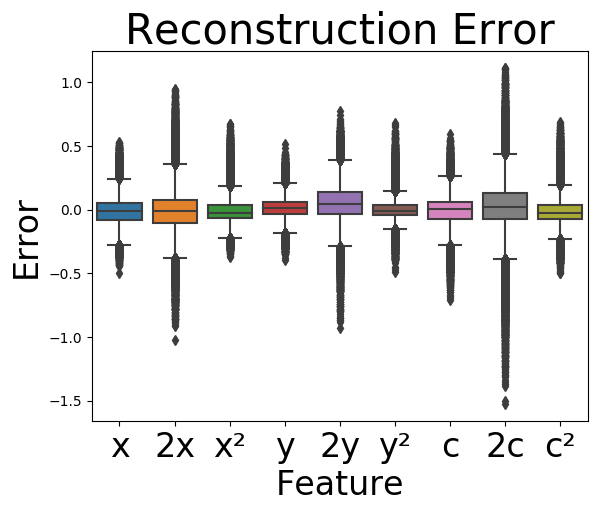}
&
\includegraphics[width=1.7in]{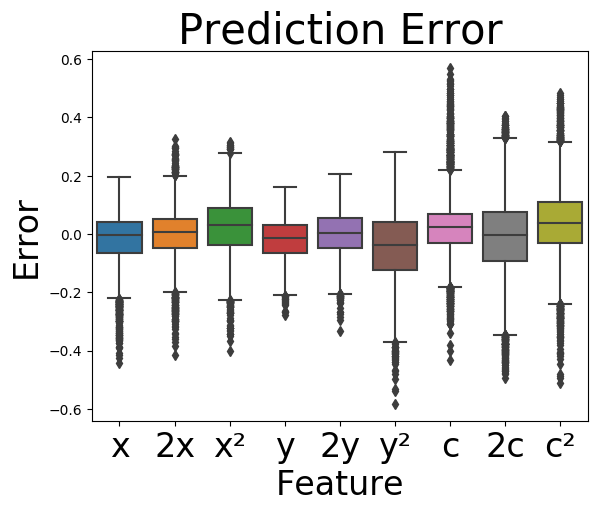}
&
\includegraphics[width=1.7in]{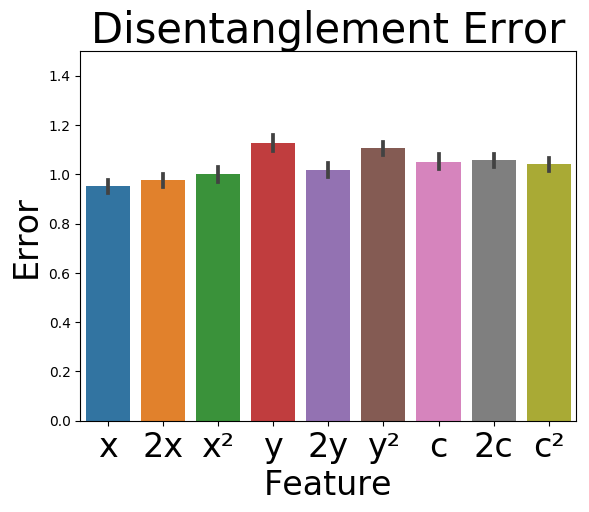}\\

\end{tabular}
\caption{Errors on the synthetic $x+y$ data for the reconstruction error (left) when taken across influence audits for each feature, prediction error (middle), and disentanglement error (right).}
\label{fig:synth_error}
\end{center}
\end{figure}

These influence experiments on the $x+y$ dataset demonstrate the importance of a good disentangled representation to the quality of the resulting indirect influence measures, since the handcrafted zero-error disentangled representation clearly results in more accurate influence results.  Each of the error types described above are given for the learned disentangled representation in Figure \ref{fig:synth_error}.  While most features have reconstruction and prediction errors close to 0 and disentanglement errors close to 1, a few features also have some far outlying instances.  For example, we can see that the $c, 2c, \mbox{ and } c^2$ variables have high prediction error on some instances, and this is reflected in the incorrect indirect influence that they're found to have on the learned representation for some instances.

\subsection{dSprites Image Classification}

\begin{wrapfigure}{r}{0.5\textwidth}
\begin{center}
  \textbf{Indirect Influence}\\
\includegraphics[width=2.5in]{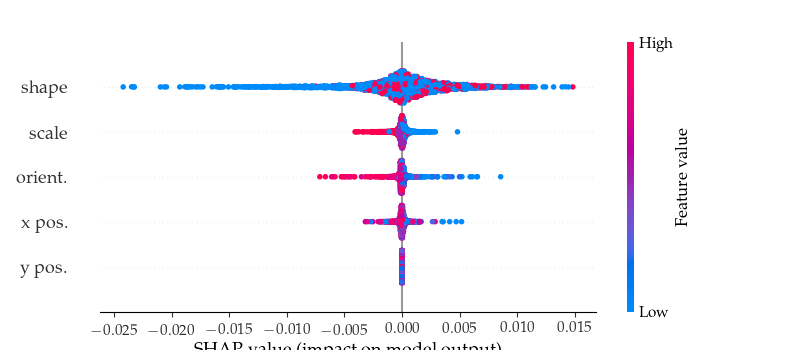}\\
\caption{dSprites data indirect latent factor influences on a model predicting shape.}
\label{fig:sprites}
\end{center}
\end{wrapfigure}

The second synthetic dataset is the \textbf{dSprites} dataset commonly used in the disentangled representations literature to disentangle independent factors that are sources of variation \cite{dsprites17}.  The dataset consists of $737,280$ images ($64 \times 64$ pixels) of a white shape (a square, ellipse, or heart) on a black background.  The independent latent factors are $x$ position, $y$ position, orientation, scale, and shape. The images were downsampled to $16 \times 16$ resolution and only the half of the data in which the shapes are largest were used due to the lower resolution.
The binary classification task is to predict whether the shape is a heart. A good disentangled representation should be able to separate the shape from the other latent factors. 

\begin{figure}[htbp]
\begin{center}
\begin{tabular}{ccc}

\includegraphics[width=1.5in]{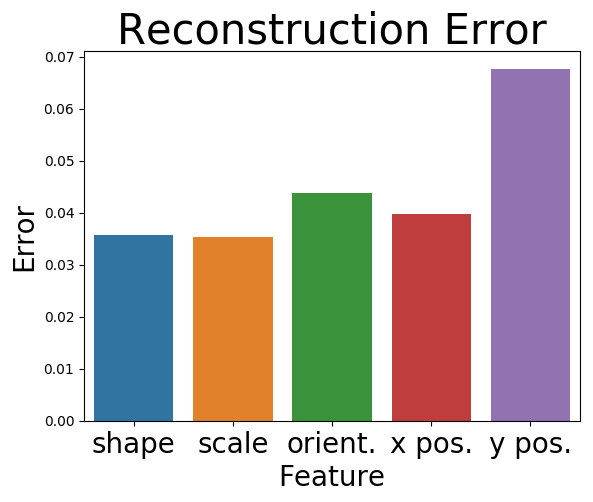}
&
\includegraphics[width=1.5in]{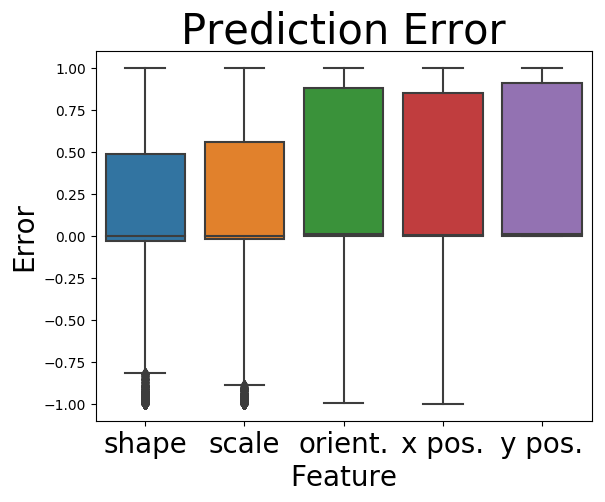} 
&
\includegraphics[width=1.5in]{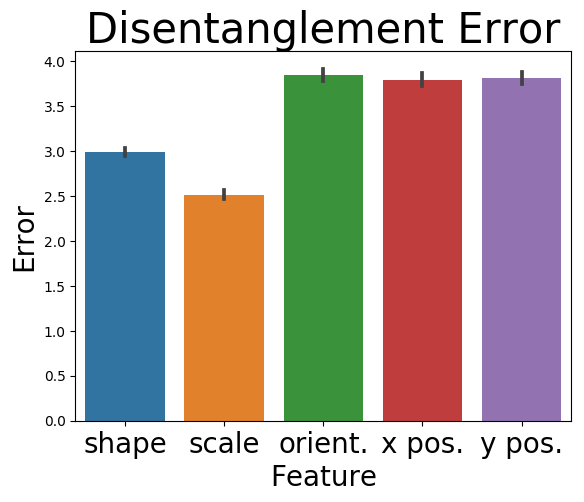}\\

\end{tabular}
\caption{The mean squared reconstruction error (left), absolute prediction error (middle), and absolute disentanglement error (right) of the latent factors in the dSprites data under an indirect influence audit.}
\label{fig:dsprites_error}
\end{center}
\end{figure}

In this experiment we seek to quantify the indirect influence of each latent factor on a model trained to predict the shape from an image. Since \texttt{shape} is the label and the latent factors are independent, we expect the feature \texttt{shape} to have more indirect influence on the model than any other latent factor. Note that a direct influence audit is impossible since the latent factors are not themselves features of the data.  Model and disentangled representation training information can be found in Appendix A.

The indirect influence audit, shown in Figure \ref{fig:sprites}, correctly identifies \texttt{shape} as the most important latent factor, and also correctly shows the other four factors as having essentially zero indirect influence.  However, the audit struggles to capture the extent of the indirect influence of \texttt{shape} since the resulting \shap\ values are small.

The associated error measures for the dSprites influence audit are shown in Figure \ref{fig:dsprites_error}.  We report the reconstruction error as the mean squared error between $x$ and $\hat {x}$ for each latent factor.
 The prediction error is the difference between $x$ and $\hat x$ of the model's estimate of the probability the shape is a heart.  While the reconstruction errors are relatively low (less than 0.05 for all but $y$ position) the prediction error and disentanglement errors are high. A high prediction error indicates that the model is sensitive to the errors in reconstruction and the indirect influence results may be unstable, which may explain the low \shap\ values for shape in the indirect influence audit.

\subsection{Adult Income Data}

\begin{figure}[htb]
\begin{center}
\includegraphics[height=1.7in]{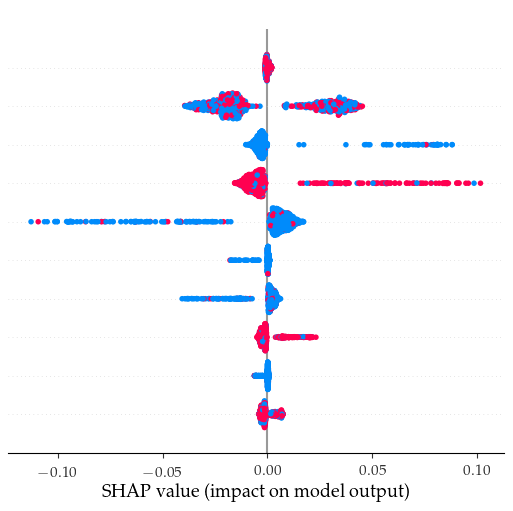}
~~
\includegraphics[height=1.7in]{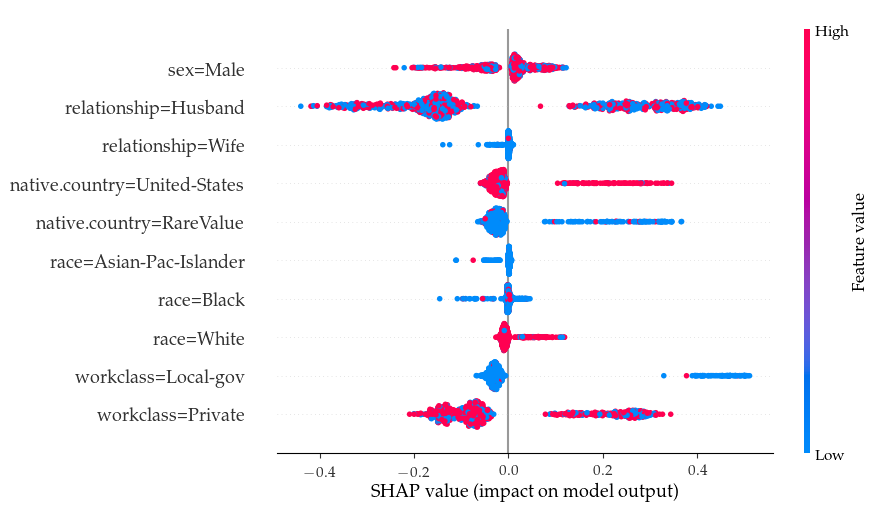}
\caption{Ten selected features for Adult dataset. Direct (left) and indirect (right) influence are shown. For all features, see Supplemental Material. Low values indicate a one-hot encoded feature is \texttt{false}.}
\label{fig:adult}
\end{center}
\end{figure}

Finally, we'll consider a real-world dataset containing \textbf{Adult Income} data that is commonly used as a test case in the fairness-aware machine learning community.  The Adult dataset includes 14 features describing type of work, demographic information, and capital gains information for individuals from the 1994 U.S. census \cite{adultData}.  The classification task is predicting whether an individual makes more or less than \$50,000 per year.  Preprocessing, model, and disentangled representation training information are included in Appendix A.

Direct and indirect influence audits on the Adult dataset are given in Figure \ref{fig:adult} and in Appendix B.  While many of the resulting influence scores are the same in both the direct and indirect cases, the \ourtechnique\ finds substantially more influence based on \texttt{sex} than the direct influence audit - this is not surprising given the large influence that sex is known to have on U.S. income.  Other important features in a fairness context, such as nationality, are also shown to have indirect influences that are not apparent on a direct influence audit.  The error results (Figure \ref{fig:adult_error} and Appendix B) indicate that while the error is low across all three types of errors for many features, the disentanglement errors are higher (further from 1) for some rare-valued features.  This means that despite the indirect influence that the audit did find, there may be additional indirect influence it did not pick up for those features.

\begin{figure}[htbp]
\begin{center}
\begin{tabular}{ccc}
\includegraphics[width=1.7in]{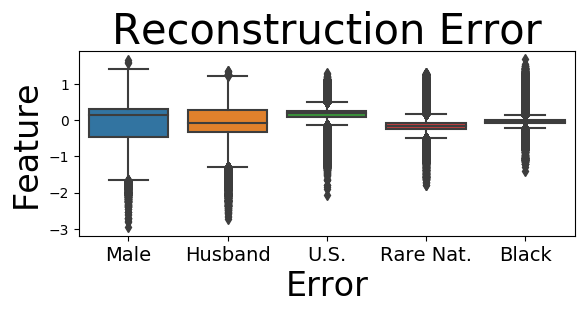}
&
\includegraphics[width=1.7in]{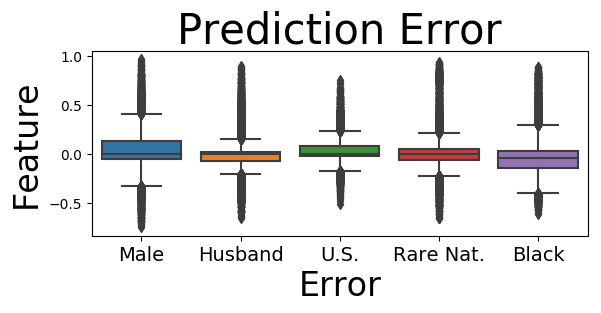} 
&
\includegraphics[width=1.7in]{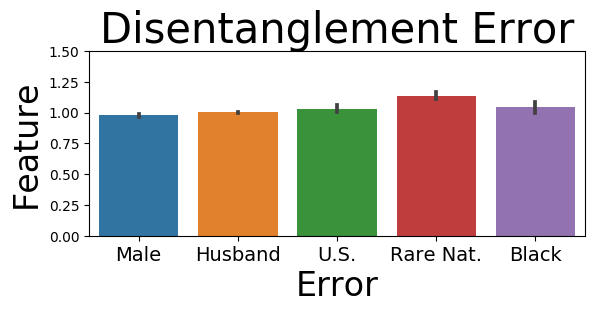}\\

\end{tabular}
\caption{The reconstruction error (left), prediction error (middle), and disentanglement error (right) of selected Adult Income features under an indirect influence audit;  see the supplemental material for the complete figure.}
\label{fig:adult_error}
\end{center}
\end{figure}

\subsection{Comparison to Other Methods}

\begin{figure}[htb]
\begin{center}
\begin{tabular}{ccc}
\includegraphics[width=1.5in]{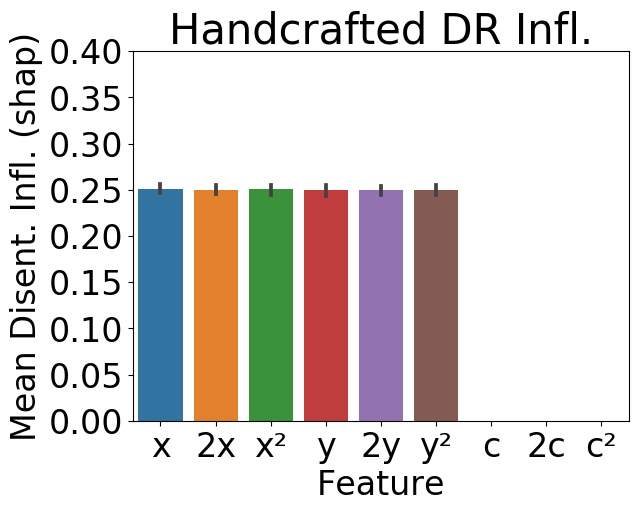}
&
\includegraphics[width=1.5in]{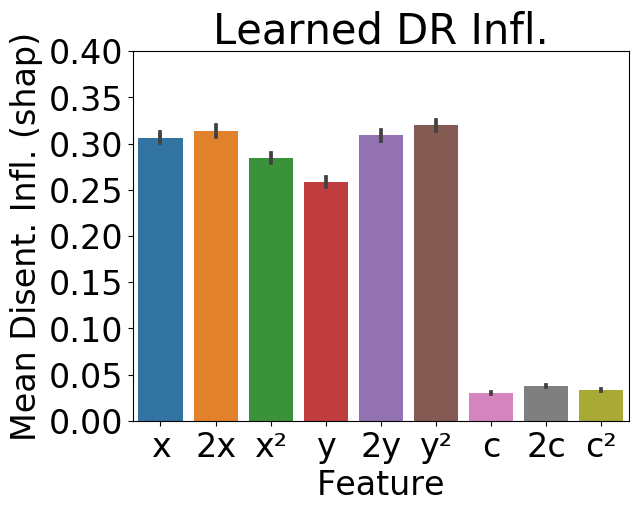}
&
\includegraphics[width=1.5in]{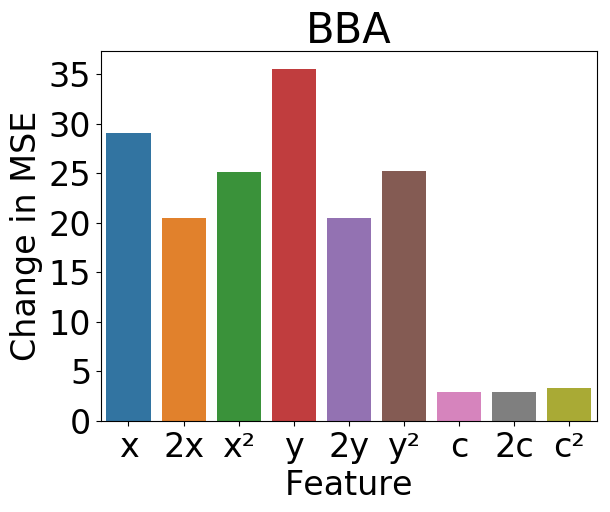}
\end{tabular}
\caption{Comparison on the synthetic $x+y$ data of the \ourtechnique\ using the handcrafted (left) or learned (middle) disentangled representation with the \bba\ approach of \cite{adler2018auditing} (right).}
\label{fig:synth_comparison}
\end{center}
\end{figure}

Here, we compare the \ourtechnique\ results to results on the same datasets and models by the indirect influence technique introduced in \cite{adler2018auditing}, which we will refer to as \bba\ (black-box auditing).\footnote{This method is available via \texttt{pip install BlackBoxAuditing}.  See also: \url{https://github.com/algofairness/BlackBoxAuditing}}  However, this is not a direct comparison, since \bba\ is not able to determine feature influence for individual instances, only influence for a feature taken over all instances.
 In order to compare to our results, we will thus take the mean over all instances of the absolute value of the per feature disentangled influence.  \bba\ was designed to audit classifiers, so in order to compare to the results of \ourtechnique\ we will consider the obscured data they generate as input into our regression models and then report the average change in mean squared error for the case of the synthetic $x+y$ data.
 (\bba\ can't handle dSprites image data as input.)

\begin{wrapfigure}{r}{0.5\textwidth}
\begin{center}
\includegraphics[width=1.7in]{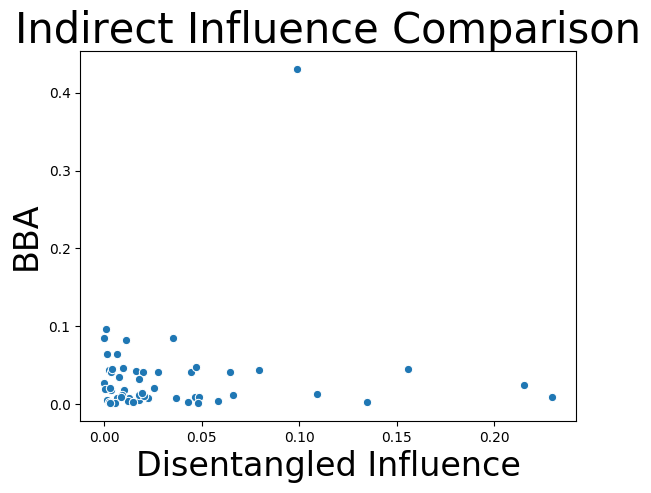}
\caption{Comparison on the Adult data of the \ourtechnique\ versus the \bba\ indirect influence approach of \cite{adler2018auditing}.}
\label{fig:adult_comparison}
\end{center}
\end{wrapfigure}

A comparison of the disentangled influence and \bba\ results on the synthetic $x+y$ data shown in figure \ref{fig:synth_comparison} shows that all three variants of indirect influence are able to determine that the $c, 2c, c^2$ variables have comparatively low influence on the model.  The disentangled influence with a handcrafted disentangled representation shows the correct indirect influence of each feature, while the learned disentangled representation influence is somewhat more noisy, and the BBA results suffer from relying on the mean squared error (i.e., the amount of influence changes based on the feature's value).

Figure \ref{fig:adult_comparison} shows the mean absolute disentangled influence per feature on the x-axis and the \bba\ influence results on the y-axis.  It's clear that the \ourtechnique\ technique is much better able to find features with possible indirect influence on this dataset and model: most of the \bba\ influences are clustered near zero, while the disentangled influence values provide more variation and potential for insight.

\section{Discussion and Conclusion}

In this paper, we introduce the idea of disentangling influence: using the ideas from disentangled representations to allow for indirect influence audits.  We show via theory and experiments that this method works across a variety of problems and data types including classification and regression as well as numerical, categorical, and image data.  The methodology allows us to turn any future developed direct influence measures into indirect influence measures.  In addition to the strengths of the technique demonstrated here, \ourtechnique\ have the added potential to allow for multidimensional indirect influence audits that would, e.g., allow a fairness audit on both race and gender to be performed (without using a single combined race and gender feature \cite{friedler2019comparative}).  We hope this opens the door for more nuanced fairness audits.

{\small 
\bibliographystyle{abbrv}
\bibliography{influence}
}
\newpage
\appendix
\section{Implementation Details}
\label{impDetails}
\paragraph{Synthetic $x+y$ model and disentangled representation information.} In both our synthetic experiments with handcrafted and trained disentangled representations we audit a model with no hidden layers that computes $x+y$ exactly from the features $x$ and $y$. 

The handcrafted disentangled representation is created to map the features with no error.  Suppose for example the protected feature, denoted $p$, was one of the features based on $y$ (one of $y$, $2y$, $y^2$). The disentangled representation used in this case would be $([x,c],[p])$. Here, we see that p will fully reveal the information relating to all of the features based on y, and $X'=[x,c]$ does not reveal any information about the protected feature. Thus, this representation satisfies the independence and preservation of information requirements. The decoder then maps this vector back to the original feature vector $(x, 2x, x^2, y, 2y, y^2, z, 2z, z^2)$, in the natural way. If for example $p=y^2$, the decoder first computes $\sqrt{p}$ to calculate $y$, then uses this to compute $2y$. All features relating to $x$ and $z$ are computed from $x$ and $z$ in the natural way as well.

In the disentangled representation we train the encoder, decoder and discriminator each have two hidden layers of 10 hidden units each. We use a 4 dimensional latent vector. All layers in each model have ReLU activations except for the last layer of the decoder and discriminator which have sigmoid activations. We use $\beta=0.5$ as the importance of disentanglement for the encoder. The minibatch size is 16 and we optimize for 10,000 train steps using $\texttt {SGD}$ with a constant learning rate of 0.01. 

\paragraph{dSprites model and disentangled representation information.}
The model we use to predict the shape from the image is a neural network with three layers of 128, 64, and 32 hidden units respectively, and achieves a $97\%$ prediction accuracy on a held out test set. The test set was randomly drawn as $20\%$ of the data. To generate the disentangled representation we use an encoder, decoder and discriminator each with a single hidden layer of 256, 256 and 64 hidden units respectively. We use a 16 dimensional latent vector. The minibatch size is 100 and we optimize for 10,000 train steps using $\texttt {SGD}$ with a constant learning rate of 0.05. All layers in each model have ReLU activations except for the last layer of the decoder and discriminator which have sigmoid activations. We use $\beta=1$ as the importance of disentanglement for the encoder. 

\paragraph{Adult Income preprocessing, model, and disentangled representation information.}

During preprocessing, categorical features are one-hot encoded and numerical features are normalized to mean 0 and standard deviation 1. The ``education\_num" feature is dropped during preprocessing. For each categorical feature, values which occur in less than 1,000 instances are binned into ``rare\_value". We train a classifier for the ``income>=50K" label with binary cross entropy loss and no hidden layers. The classifier achieves test loss of 0.326 and test accuracy of $84.9\%$.

To generate the disentangled representation we use an encoder, decoder and discriminator which each have two hidden layers with 25 and 12 hidden units respectively. We use a 10 dimensional latent vector. We use $\beta=0.5$ as the importance of disentanglement for the encoder. The models are trained for 4000 train steps with minibatch sizes of 16, using $\texttt {SGD}$ with a constant learning rate of 0.01. We used the canonical train/test split. 

\paragraph{Additional Information.}
All models for the synthetic $x+y$ and dSprites experiments were trained on a MacBook Pro (Early 2015) with a 2.7GHz Processor and 8 GB of RAM. The models for the adult experiments were trained on an NVIDIA Titan Xp GPU. Hyperparameters were chosen via experimentation. Only architectures containing 2 or fewer hidden layers were considered for models used to disentangle the data. The minibatch sizes tested were between 16 and 100, and learning rates between 0.01 and 0.1 were tested. In each experiment, we used at least 5 and no more than 15 evaluation runs. 

\newpage
\section{Full Results for Adult Income Dataset}
\subsection{Direct and Indirect Influence Results}
\begin{figure}[htb]
\begin{center}
\includegraphics[width=2.5in]{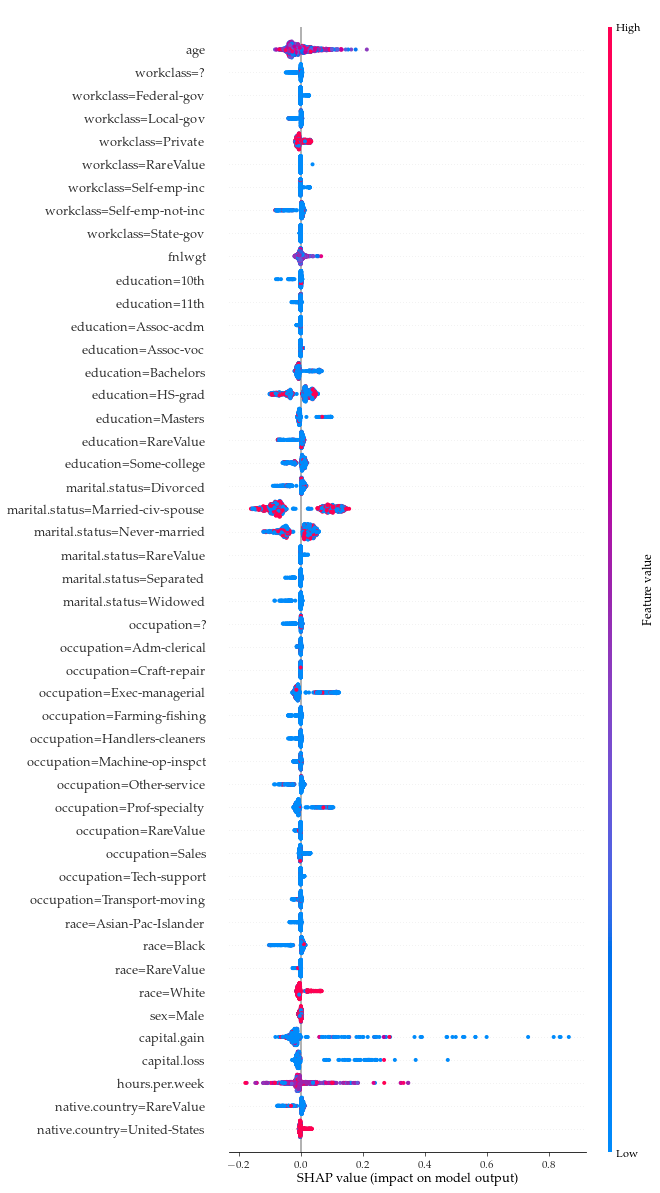}
~~
\includegraphics[width=2.5in]{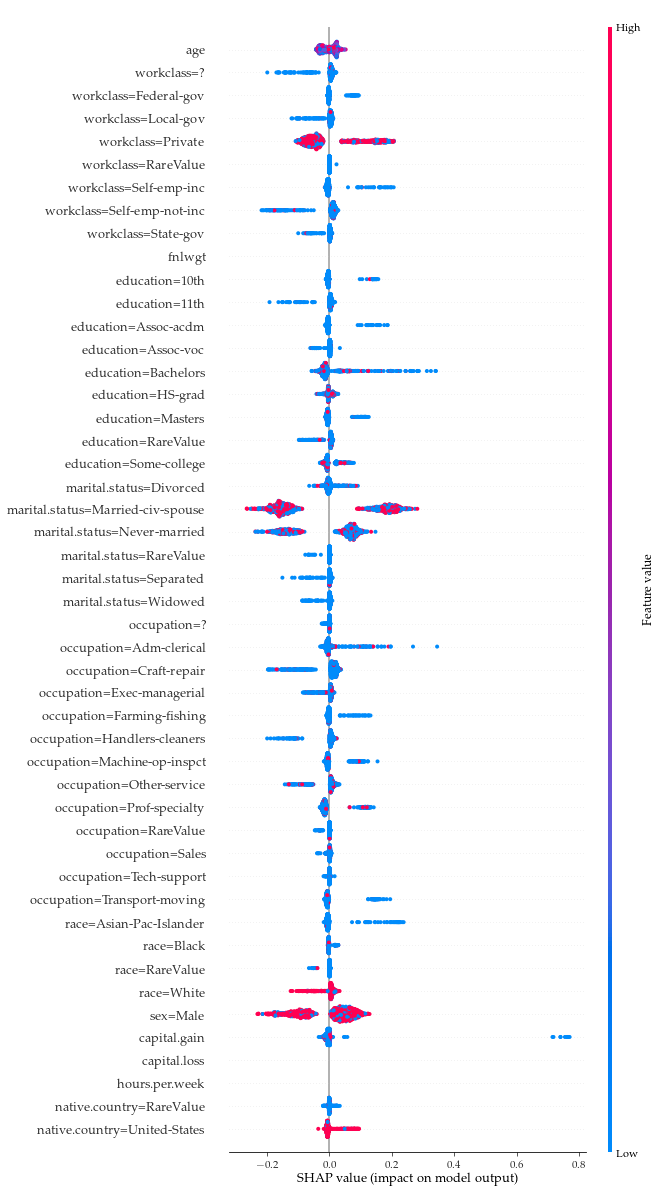}
\caption{The full influence results for the adult data direct (left) and indirect (right) feature influences.}
\label{fig:adult}
\end{center}
\end{figure}

\newpage
\subsection{Error Results}
\begin{figure}[htb]
\begin{center}
\includegraphics[width=2.5in]{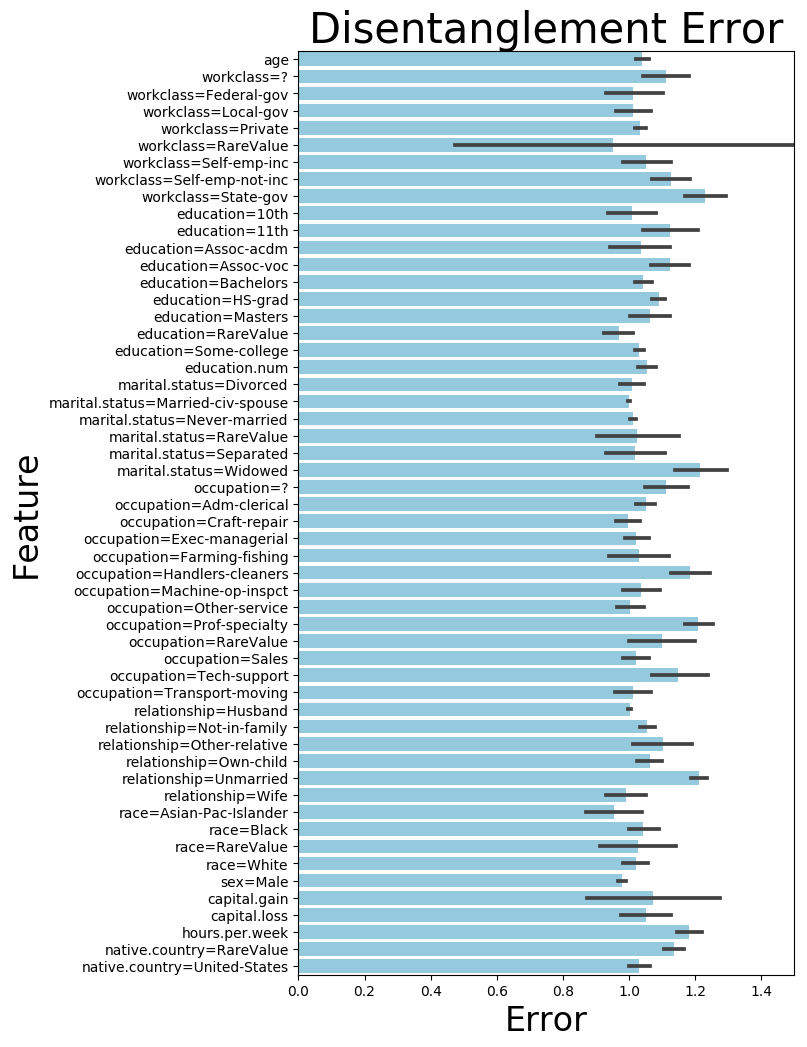} \\
\end{center}
\end{figure}

\begin{figure}[htb]
\begin{center}
\includegraphics[width=2.5in]{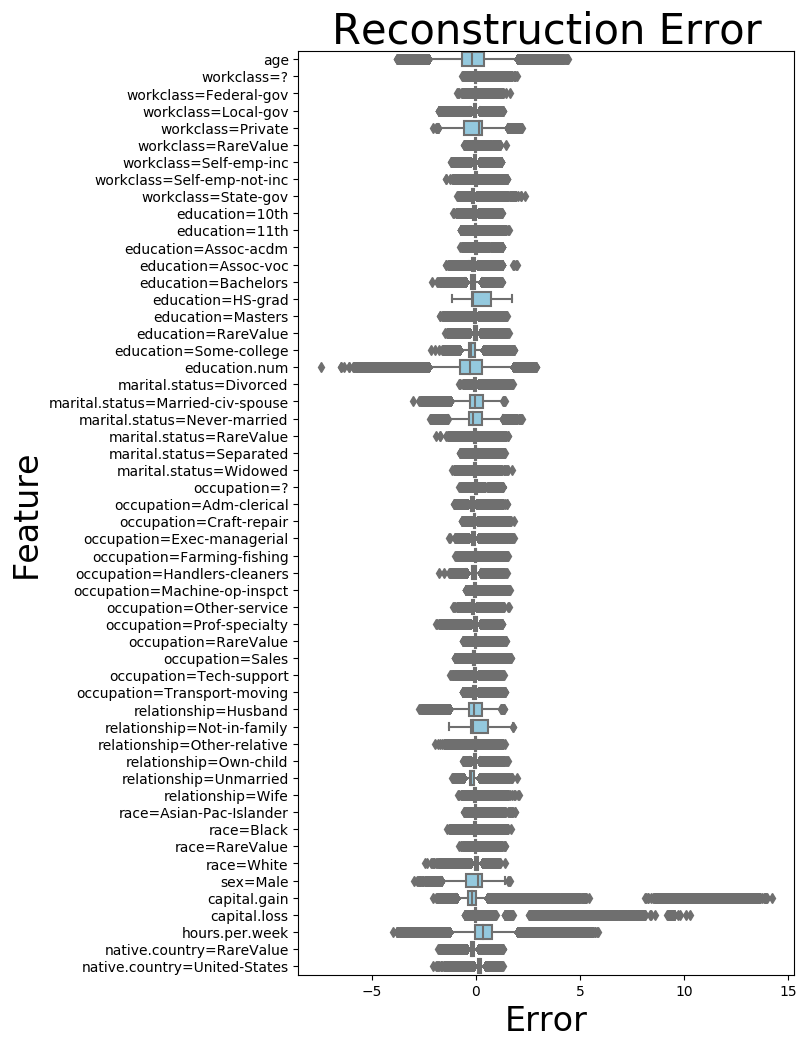}
\includegraphics[width=2.5in]{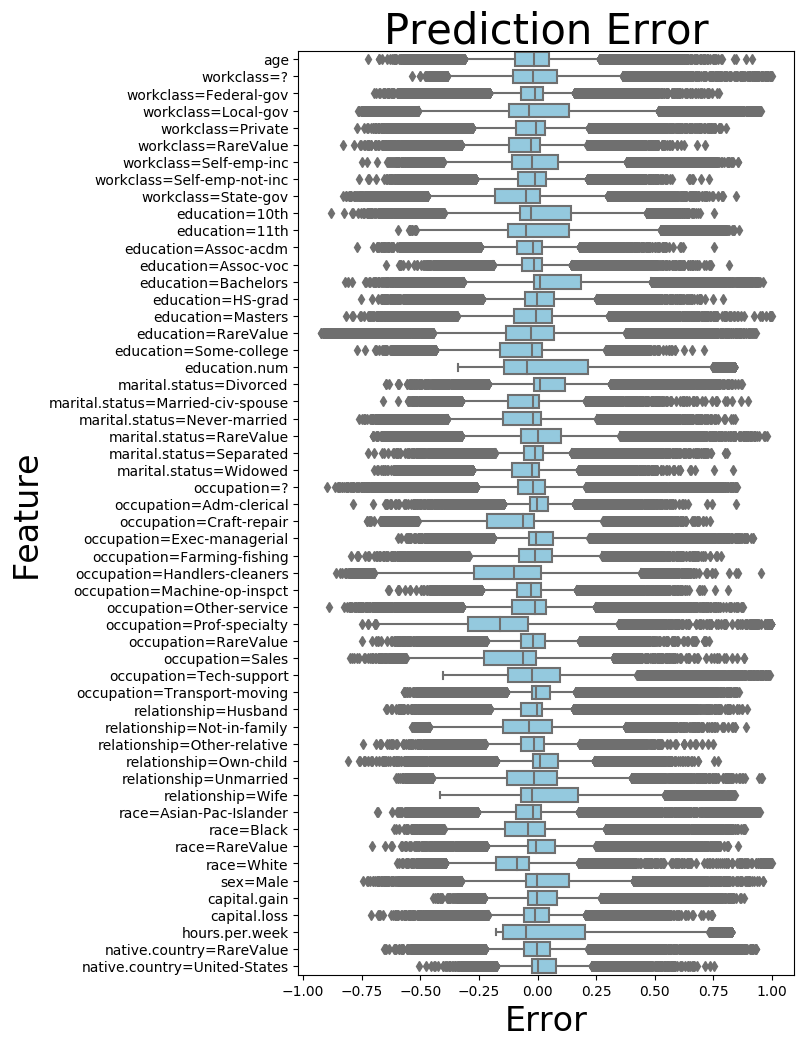}
 \\
 \caption{The full disentanglement (top), reconstruction (left) and prediction (right) error metrics for the adult data experiment.}
\end{center}
\end{figure}

\end{document}